\documentclass[11pt]{article}
\usepackage[a4paper,total={6.5in,9.5in}]{geometry}

\usepackage[utf8]{inputenc} 
\usepackage[T1]{fontenc} 
\usepackage{hyperref}
\hypersetup{
    colorlinks,
    linkcolor={red!50!black},
    citecolor={blue!50!black},
    urlcolor={blue!80!black},
}

\usepackage{url}       
\usepackage{booktabs}       
\usepackage{amsfonts}      
\usepackage{nicefrac}  
\usepackage{microtype} 
\PassOptionsToPackage{usenames,dvipsnames}{xcolor} 

\usepackage{amsmath,amssymb,amsthm}
\usepackage{thm-restate}
\usepackage[capitalise]{cleveref}
\usepackage{wrapfig, graphicx}
\usepackage{xargs}
\usepackage[colorinlistoftodos,prependcaption,textsize=tiny]{todonotes}
\usepackage{selectp}

\usepackage{caption,subcaption}
\usepackage{enumitem}
\usepackage{svg}
\usepackage{bbold}

\def\R{\mathbb{R}}

\DeclareMathOperator*{\bbP}{\mathbb{P}}

\def\cB{\mathcal{B}}

\def\cS{\mathcal{S}}

\def\cX{\mathcal{X}}
\def\cY{\mathcal{Y}}

\def\ie{\textit{i.e.}}

\DeclareMathOperator*\E{\mathbb{E}}

\def\1{\mathbf{1}}
\def\0{\mathbf{0}}

\newcommand{\eg}{\textit{e.g.}}

\DeclareMathOperator*{\argmin}{arg\,min}
\DeclareMathOperator*{\argmax}{arg\,max}

\newcommand{\rom}[1]{%
  \textup{\uppercase\expandafter{\romannumeral#1}}%
}

\newtheorem{example}{Example}[section]

\crefformat{footnote}{#2\footnotemark[#1]#3}
\newcommandx{\ambar}[2][1=]{\todo[linecolor=orange,backgroundcolor=orange!25,bordercolor=orange,#1]{AP: #2}}

\usepackage{amsmath}
\usepackage{amssymb}
\usepackage{mathtools}
\usepackage{amsthm}
\usepackage{makecell}
\usepackage{multirow}
\usepackage{natbib}

\theoremstyle{plain}
\newtheorem{theorem}{Theorem}[section]

\newtheorem{lemma}[theorem]{Lemma}

\theoremstyle{definition}
\newtheorem{definition}[theorem]{Definition}

\theoremstyle{remark}

\newcommand{\vol}[1]{{\rm Vol} \left( #1 \right)}
\newcommand{\red}[1]{{#1}}

\title{Certified Robustness against Sparse Adversarial Perturbations via Data Localization}
\date{}

\author{
    Ambar~Pal\footnote{Corresponding author: ambar@jhu.edu}~\footnote{Department of Computer Science, Johns Hopkins University, Baltimore, MD, USA}~\footnote{Mathematical Institute for Data Science (MINDS), Johns Hopkins University, Baltimore, MD, USA} 
    \and René~Vidal\footnote{Center for Innovation in Data Engineering and Science (IDEAS), University of Pennsylvania, Philadelphia, USA} 
    \and Jeremias~Sulam\footnotemark[3]~\footnote{Department of Biomedical Engineering, Johns Hopkins University, Baltimore, MD, USA}
}

\begin{document}
\maketitle

%\begin{icmlauthorlist}
%\icmlauthor{Ambar Pal}{jhucs}
%\icmlauthor{René Vidal}{upenn}
%\icmlauthor{Jeremias Sulam}{jhubme}
%\end{icmlauthorlist}

%\icmlaffiliation{jhucs}{Department of Computer Science \& Mathematical Institute for Data Science (MINDS), Johns Hopkins University, Baltimore, MD 21210, USA}
%\icmlaffiliation{upenn}{Center for Innovation in Data Engineering and Science (IDEAS), University of Pennsylvania, Philadelphia, USA}
%\icmlaffiliation{jhubme}{Department of Biomedical Engineering \& Mathematical Institute for Data Science (MINDS), Johns Hopkins University, Baltimore, MD, 21210, USA }

\begin{abstract}
Recent work in adversarial robustness suggests that natural data distributions are localized, \ie, they place high probability in small volume regions of the input space, and that this property can be utilized for designing classifiers with improved robustness guarantees for $\ell_2$-bounded perturbations. Yet, it is still unclear if this observation holds true for more general metrics. In this work, we extend this theory to $\ell_0$-bounded adversarial perturbations, where the attacker can modify a few pixels of the image but is unrestricted in the magnitude of perturbation, and we show necessary and sufficient conditions for the existence of $\ell_0$-robust classifiers. Theoretical certification approaches in this regime essentially employ voting over a large ensemble of classifiers. Such procedures are combinatorial and expensive or require complicated certification techniques. In contrast, a simple classifier emerges from our theory, dubbed \textsc{Box-NN}, which naturally incorporates the geometry of the problem and improves upon the current state-of-the-art in certified robustness against sparse attacks for the MNIST and Fashion-MNIST datasets.
\end{abstract}
\section{Introduction}
It is by now well known that adversarial attacks affect Machine Learning (ML) systems that can potentially be used for security sensitive applications. However, despite significant efforts on robustifying ML models against adversarial attacks, it has been observed that their performance on most tasks under adversarial perturbation is not close to human levels. This motivated researchers to obtain theoretical impossiblity results for adversarial robustness \cite{shafahi2018inevitable,dohmatob,dai2022fundamental}, which state that for general data distributions, no robust classifier exists against adversarial perturbations, even when the adversary is limited to making small $\ell_p$-norm-bounded perturbations. However, such results are seemingly in conflict with the fact that humans can classify most natural images quite well under small $\ell_p$-norm-bounded perturbations. Even more, there is a rich literature on certified robustness, \eg, \cite{zhang2018efficient,cohen2019certified,pal2020game,fischer2020certified,jeong2020consistency,jia2022almost,pfrommer2023projected,salman2022certified,eirasancer,pal2023understanding}, where the goal is to obtain and analyze methods with provable guarantees on their robustness under adversarial attacks.

\citet{pal2023concentration} recently provided a solution to this apparent conflict, noting that existing impossibility results become vacuous when the data distribution is such that a large probability mass is concentrated on very small volume in the input space, a property they call $(C, \epsilon, \delta)$-concentration. This characterization implies that at least $1 - \delta$ probability mass is found in a region of volume at most $C e^{-n\epsilon}$ for small $\delta \approx 0$ and large $\epsilon$. As an example, this property dictates that sampling a random $224 \times 224$ dimensional image is extremely likely to \emph{not} be a natural image. This property is intuitively satisfied for natural datasets like ImageNet, and \citet{pal2023concentration} formally show that whenever a classifier robust against small $\ell_2$-bounded attacks exists for a data distribution (e.g., humans for natural images), this distribution must be concentrated. This shows that indeed, robust classifiers against $\ell_2$ attacks can be obtained for natural image distributions, and there is no impossibility. 

%\todo{JS: Is stronger to point out problems in prior work instead of working due to interest} 
While these results are encouraging, attacks that are bounded in Euclidean norm have nice analytical properties that facilitated the results in \citet{pal2023concentration}. In this work, we seek to understand if similar notions can provide insights on provable defenses against \emph{sparse} adversarial attacks (bounded in their $\ell_0$ distance) where the adversary is limited to modifying a few pixels on the image, but those pixels can be modified in an unbounded fashion. 
Even though for humans it seems trivial to correctly classify a natural image corrupted in a few pixels, this problem has stood out as a particularly hard task for machine learning models. The difference is extreme: \citet{su2019one} demonstrated that adversarially modifying a \emph{single} pixel leads to large performance degradation of many state of the art image recognition models. Standard ideas for improving robustness, like adversarial training, seem to be empirically ineffective against sparse attacks. %{\color{blue} AP: citation needed} .%, and it has been shown that networks with ReLU activations are susceptible to $\ell_0$ perturbations. 
Since then, researchers have resorted to enumerating a large number of subsets of the input pixels, and taking a majority vote over the class predicted from each subset, as a means of obtaining classifiers robust to sparse attacks. The resultant methods \citep{levine2020robustness} are expensive, and need probabilistic certificates due to the combinatorial blow-up in the number of subsets needed as the number of attacked pixels increases. Follow-up work by \citet{jia2022almost} has employed complicated certification schemes to reduce the slack in these certificates, while still remaining computationally expensive. Most recently, \citet{hammoudeh2023feature} carefully selected these subsets to speed up the certificate computation. However, none of these existing methods utilizes the geometry of the underlying data distribution highlighted by our results. Departing from this stream of research, we propose a classifier that closely utilizes this underlying geometry to obtain 
%\emph{tight} 
robustness certificates. As a result, we provide a classifier that is lighter and simpler than all existing works, and an associated certification algorithm with  %tight 
$\ell_0$ certificates that are better than prior work.
%\footnote{\color{red}The purpose of this paragraph is not clear.} This paragraph has been relocated to make more sense.

Our proof techniques extend results in \citet{pal2023concentration} to sparse adversarial attacks.
In practice, one can always project the pixel values to lie in some predefined range, say $[0, 1]$, before classification, so we can consider adversarial perturbations to lie within $[0, 1]^n$ without any loss of generality. In other words, our adversary at power $\epsilon$ is allowed to modify an image from $x$ to $x'$ such that $\|x - x'\|_0 \leq \epsilon, \|x'\|_\infty \leq 1$. The techniques in \citet{pal2023concentration} break down under such an adversary, as their first assumption is to restrict attention to adversarial perturbations $v$ such that $x + v$ cannot lie $\epsilon$-close to the boundary of the image domain. In our case, the geometry of the problem is radically different: even a perturbation of size $1$ is sufficient to take any image to the boundary of the domain $[0, 1]^n$ (simply perturb any pixel to $1$). 
%\footnote{\color{red}RV: A more precise statement would be that it must lie in the interior AP: Resolved -- previous paper required $\epsilon$-close to boundary.} 
As a result, although we are motivated by \citet{pal2023concentration}, our theory and certification algorithms are markedly different from those in that work.

In the above setting, we show that whenever there exists a classifier robust to adversarial modification of a few entries in the input, the underlying data distribution \red{places a large mass, i.e., \emph{localizes},} on low-volume subsets of the input space. We further show that the converse holds too, albeit with a strengthening of the \red{localization} condition; i.e., we show that when the data distribution \red{localizes} on low-volume subsets of the input space, and these subsets are sufficiently separated from one another, then a robust classifier exists. These results suggest that such underlying geometry in natural image distributions should be exploited for constructing classifiers robust against $\ell_0$ attacks. Indeed, we then propose a simple
classifier, called Box Nearest Neighbors (\textsc{Box-NN}), that utilizes this underlying geometry by having decision regions that are unions of axis-aligned rectangular boxes in the input space. Such a classifier naturally allows for $\ell_0$ robustness certificates that improve upon prior work for certified defenses in a wide regime. %However, to the best of our knowledge, no $\ell_0$ defense (certified or otherwise) in the literature considers any geometry of the underlying data distributions. 

To summarize, we make the following contributions in this work:
\begin{enumerate}
\item In \cref{sec:necessary} we show that if a data-distribution $p$ defining a multi-class classification problem admits a robust classifier whose error is at most $\delta$ under sparse adversarial perturbations to $\epsilon$ pixels, then %at least one class conditional $p_{X | Y = k}$ places 
there is a subset $S$ of volume at most $C e^{-\epsilon^2/n}$ and a class $k$ such that the class conditional $q_k$ places a large mass $q_k(S) \geq 1 - \delta$ on $S$, \ie, 
%{\color{red}$x \sim p$ lies in $S$ with probability at least $1 - \delta$,} \ie, 
$q_k$ is $(C, \epsilon^2 / n, \delta)$-\red{localized}.

\item In \cref{sec:sufficient}, we show that a stronger notion of \red{localization}, which ensures that the class conditional distributions are sufficiently separated with respect to the $\ell_0$ distance, is sufficient for the existence of a robust classifier. In fact, this result generalizes to any distance $d$, showing the existence of a robust classifier w.r.t.~perturbations bounded in distance $d$ whenever the data distribution $p$ is strongly \red{localized} with respect to $d$.

\item In \cref{sec:boxclassifier}, we propose a classifier certifiably robust against sparse adversarial attacks, called \textsc{Box-NN}, and derive certificates of $\ell_0$ robustness for it.
%, and show that these certificates are tight. 
We then provide empirical evaluation on the MNIST and the Fashion-MNIST datasets, and demonstrate that \textsc{Box-NN} obtains state-of-the-art results in certified $\ell_0$ robustness.%\footnote{\color{red}RV: Should Box-NN be mentioned in title/abstract?} AP: Yes, this added to abstract.
\end{enumerate}

\section{Existence of an $\ell_0$-Robust Classifier implies \red{Localization}} \label{sec:necessary}
%\section{$\ell_0$-Robustness implies Concentration} \label{sec:l0conc}
We will take our data domain to be $[0, 1]^n$, to mimic the standard natural image classification tasks\footnote{Albeit with a scaling -- natural images are typically stored with each pixel value in $[0, 255]$.}, \ie, $\cX = \{x \colon \|x\|_\infty \leq 1\}$. We will take our label domain to be $\cY = \{1, 2, \ldots, K\}$, and assume that we have a classification task defined by a data distribution $p$ over $\cX \times \cY$. The conditional distribution $p_{X | Y = k}$ for each $k \in \cY$ will be denoted by $q_k$. 

For any classifier $f \colon \cX \to \cY$, we recall the standard definition of robust risk $R_{d}(f, \epsilon)$ against perturbations bounded in a distance $d$ as 
\begin{equation*}
R_{d}(f, \epsilon) = \bbP_{(x, y) \sim p} \left( \exists \bar x \in B_{d}(x, \epsilon) \text{ such that } f(\bar x) \neq y \right). %\label{robustness-defn}
\end{equation*}
Similarly, we define a classifier $f$ to be $(\epsilon, \delta)$-robust with respect to a distance $d$ if the robust risk against perturbations at a distance bounded by $\epsilon$ is at most $\delta$, \ie, $R_d(f, \epsilon) \leq \delta$.

For the rest of this section, we will assume that $p$ defines a task for which one can obtain a classifier $f$ such that $R_{\ell_0}(f, \epsilon) \leq \delta$, where  $\epsilon$ is a non-negative integer denoting the maximum number of pixels that an adversary can perturb. Given such an $f$, we will show that $p$ should satisfy the special property of \red{localization}. In other words, we will obtain a necessary condition for $\ell_0$ robustness. This special property of $(C, \epsilon, \delta)$-\red{localization} is similar to \citet[Definition 2.2]{pal2023concentration}, with a slight modification:% to defined as follows.
\begin{definition}[Localized Distribution, modification of \citet{pal2023concentration}]
\label{l0concdefn}
A probability distribution $q$ over a domain $\cX \subseteq \R^n$ is said to be $(C, \epsilon, \delta)$-\red{localized} if there exists a subset $S \subseteq \cX$ such that $q(S) \geq 1 - \delta$ but ${\rm Vol}(S) \leq C \exp(-\epsilon)$. %for some constants $c_1, c_2 > 0$. Here, %$\delta$ is thought of close to $0$, and 
Here, ${\rm Vol}$ denotes the standard Lebesgue measure on $\R^n$, and $q(S)$ denotes the measure of $S$ under $q$.
\end{definition}
\cref{l0concdefn} is similar to \citet[Definition 2.2]{pal2023concentration} but it removes the explicit dimension of the problem, \ie, $n$, from the volume constraint. This allows one to state the results in \citet{pal2023concentration}, as well as ours, under the same definition. \red{Additionally, we rename the property from \emph{concentration} in \citet{pal2023concentration} to \emph{localization}, in order to distinguish ourselves from the well known notion of \emph{concentration of measure}. These two notions are related and, before proceeding, we compare them in more detail.

The notion of \emph{measure concentration} from high dimensional probability theory  roughly states that for a given large dimension $n$, ``a well behaved function $h$ of the random variables $Z_1, Z_2, \ldots, Z_n$ takes values close to its mean $\E h(Z_1, \ldots, Z_n)$ with high probability'' \citep{talagrand1996new}. A popular quantification of this notion %, due to Gromov and Milman, 
states that for a metric space $(\cX, d)$ and a probability distribution $q$ over $\cX$, the concentration function $\alpha$ defined as
\begin{equation}
\alpha_{q, d}(t) = \sup_{S \subseteq \cX, \ q(S) \geq 1/2} 1 - q(S^{+t}),
\end{equation}
decreases ``very fast'' with $t$, where recall that $S^{+t} = \{x \in \cX \colon d(x, S) \leq t\}$. We typically say that $q$ has the property of measure concentration if there is an exponential decay as $\alpha_{q, d}(t) \sim \exp(-\gamma t)$ for all $t \geq 0$, and some universal constant $\gamma$.  
 
In contrast, the definition of $(C, \epsilon, \delta)$-localization requires the existence of $S \subseteq \cX$ such that $q(S) \geq 1 - \delta$ and $\vol{S} \leq C \exp(-\epsilon)$. Concentration and localization are similar in the underlying message: most of the mass in $q$ is concentrated near a small region in space. However, the mathematical formalization is different, as localization does not require a fast enough rate of decay of the measure, and hence does not require an underlying metric on the space $\cX$. In order to show that a given distribution $q$ localizes, it is sufficient to provide a single instance of a set $S \subseteq \cX$ that satisfies the localization parameters. For our data domain $\cX = [0, 1]^n$, we will consider a family of probability distributions given by $q_a = {\rm Unif}([0,a]^n)$ for $a \in (0, 1]$, and comment on their localization and measure concentration parameters, to shed light into their similarities and differences.

%\textbf{Uniform }
%\item Consider the uniform distribution on the unit cube $q_1 = {\rm Unif}([0, 1]^n)$. For any $S \subseteq [0, 1]^n$, we can solve $1 - \delta \leq q(S) = \vol{S} \leq \exp(-\epsilon)$ to obtain that 
%\begin{align*}
%q \text{ is } \left(1, \log \left( \frac{1}{1 - \delta} \right), \delta \right)-\text{concentrated for any }\delta \in [0, 1].
%\end{align*} 
For any $S \subseteq [0, a]^n \subseteq \cX$, we can simplify $1 - \delta \leq q_a(S) = \frac{1}{a^n} \vol{S} \leq \frac{1}{a^n} \exp(-\epsilon)$ to obtain that 
\begin{align*}
q_a \text{ is } \left(1, \log \left( \frac{1}{1 - \delta} \right) + n \log \left(\frac{1}{a}\right), \delta \right)-\text{localized for any }\delta \in [0, 1].
\end{align*} 
From the above we can see that keeping $\delta, a < 1$ fixed, $q_a$ becomes ``more localized'' as the dimension $n$ increases. Similarly, keeping $\delta, n$ fixed, $q_a$ becomes more localized as $a$ gets closer to $0$. In this sense, the localization parameters depend on the scale of the support of the underlying distribution. 
%\item Consider the $n$-dimensional isotropic gaussian distribution $q_3 = \cN(0, \sigma I_n)$. For $S = \{x \colon \|x\|_2 \leq 1\}$, we have 

In contrast, as measure concentration depends on an underlying metric, the concentration parameters are independent of the scale of the support when the metric is invariant to scaling. As an example, for $\cX$ equipped with the hamming metric, $d_0(x, x') = \|x - x'\|_0$, the concentration function for the distribution $q_a$ can be shown to be
\begin{align}
\alpha_{q_a, d_0}(t) \leq 2 \exp \left( -\frac{t^2}{n} \right).
\end{align}

%Specifically, \citet[Theorem 2.1]{pal2023concentration} shows that $(C, n \epsilon, \delta)$-concentration is a necessary condition for $\ell_2$-robustness under \cref{l0concdefn}, and we will now show that $(C, \epsilon^2 / n, \delta)$-concentration is a necessary condition for $\ell_0$-robustness. We provide more comments after \cref{th:lzeroconc}.
%{\color{blue}RV: Is there any dependency of the definition of concentrated and the distance $d$? I suspect there is, and the strong version uses the name $d$-strongly concentrated. Should we use $d$-concentrated here? Should the exponent depend on $d$? Or should we at least sat $l_0$-concentrated? BTW: $\ell_0$ is not a distance. {\color{teal}AP: Right, $\|\cdot - \cdot \|_0$ is the distance}. Why is this definition a strengthening? I would understand it as a strengthening of more distributions satisfy the definition. Or else, if one can get robustness certificates under more general settings under this new definition, which we do not know yet. In short, the writing leads to the question: why do we need a new definition? This is not explained or mentioned in the writing.
%\color{teal} AP: Removed ``strengthening'', it is unclear what does that mean in this context. Added an explanation for introducing a modified definition.}
%

Armed with the above definition, we will now derive a necessary condition for $\ell_0$-robustness in terms of localization, by using a measure-concentration result w.r.t. the $\ell_0$ distance due to \citet{talagrand1995concentration}.}
\begin{restatable}[]{theorem}{lzeroconc}
If there exists an $(\epsilon, \delta)$-robust classifier $f$ with respect to the $\ell_0$ distance for a data distribution $p$, then at least one of the class conditionals $q_1, q_2, \ldots, q_K$ must be $(C, \epsilon^2 / n, \delta)$--\red{localized} according to \cref{l0concdefn}. Further, if the classes are balanced, then all the class conditionals are $(C_{\rm max}, \epsilon^2 / n, K \delta)$-\red{localized}. Here, $C$ and $C_{\rm max}$ are constants dependent on $f$.
\label{th:lzeroconc}
\end{restatable}
\begin{proof}
We are given a classifier $f$ which is $(\epsilon, \delta)$-robust w.r.t.~perturbations bounded in the $\ell_0$ distance. In other words, we have $R_{\ell_0}(f, s) \leq \delta$. Expanding this we get
\begin{equation*}
\sum_k \bbP  \left( \exists \bar x \in B_{\ell_0}(x, \epsilon) \text{ such that } f(\bar x) \neq k \right) \bbP (y = k) \leq \delta.
\end{equation*}
In other words, there exists a class $k'$ satisfying $q_{k'} \big( \{ x \in \cX \colon \exists \bar x \in B_{\ell_0}(x, \epsilon)$ such that $f(\bar x) \neq k' \} \big) \leq \delta$. Defining
the unsafe set for the class $k'$ as $U_{k'} = \{ x \in \cX \colon \exists \bar x \in B_{\ell_0}(x, \epsilon) \text{ such that } f(\bar x) \neq k' \}$, we have shown 
\begin{equation}
q_{k'} (U_{k'}) \leq \delta. \label{unsafeset}
\end{equation}
Define $A_{k'} \subseteq \cX$ to be the region where $f$ predicts $k'$, \ie, $A_{k'} = \{x \in \cX \colon f(x) = k'\}$. Further, for any set $Z$ define $Z^{+\epsilon}$ to be all the points in the domain $\cX$ which are at most $s$ away from $Z$ in $\ell_0$ distance, \ie, $Z^{+\epsilon} = \{x \in \cX \colon \exists \bar x \in Z \text{ such that } \|x - \bar x\|_0 \leq \epsilon\}$ %{(\color{blue} AP: Careful, the expansion definition changes in Proof 2.2)}. 
Then, we have 
\begin{align*}
U_{k'} &= \{ x \in \cX \colon \exists \bar x \text{ such that } \|x - \bar x\|_0 \leq \epsilon, f(\bar x) \neq k' \} \\
&= \{ x \in \cX \colon \exists \bar x \in (\cX \setminus A_{k'}) \text{ such that } \|x - \bar x\|_0 \leq \epsilon\} \\
&= (\cX \setminus A_{k'})^{+\epsilon}. 
\end{align*}
Now, we will use measure concentration on the unit cube from \citet[Proposition 2.1.1]{talagrand1995concentration}:
\begin{lemma}[Proposition 2.1.1 in \cite{talagrand1995concentration}] \label{talagrand}
For $B \subseteq [0, 1]^n$, ${\rm dist}(x, B) = \min_{z \in B} \|x - z\|_0$, any measure $\mu$ on $[0, 1]$, we have 
\begin{equation*}
    \bbP_{x \sim \mu^n}({\rm dist}(B, x) \geq t) \leq \frac{1}{\bbP_{x \sim \mu^n}(x \in B)} \exp(-t^2 / n).
\end{equation*}
\end{lemma}
Note that since the domain $[0, 1]^n$ has $n$-dimensional volume $1$, \ie, ${\rm Vol}([0, 1]^n) = 1$, the uniform measure of any set $\mu^n(B) = {\rm Vol}(B)$, for $B \subseteq [0, 1]^n$. Substituting $B = \cX \setminus A_{k'}$, $t = \epsilon$, $\mu = {\rm Unif}([0, 1])$, in \cref{talagrand}, we 
% {\color{blue} AP: Comment a little more to explicitly mention what to substitute in that Proposition to obtain this result}, to 
obtain 
\begin{equation*}
{\rm Vol}(\cX \setminus A_{k'})^{+\epsilon} \geq 1 - \frac{\exp(-\epsilon^2/n)}{{\rm Vol} (\cX \setminus A_{k'})}.
\end{equation*} 
Using ${\rm Vol}(\cX \setminus U_{k'}) = 1 - {\rm Vol}(\cX \setminus A_{k'})^{+\epsilon}$, we obtain 
\begin{equation}
{\rm Vol}(\cX \setminus U_{k'}) \leq \frac{\exp(-\epsilon^2/n)}{{\rm Vol} (\cX \setminus A_{k'})}. \label{unsafevol}
\end{equation}
Finally, combining \eqref{unsafeset}, \eqref{unsafevol}, and taking $S = \cX \setminus U_{k'}$, we have
\begin{align*}
q_{k'}(S) \geq 1 - \delta, \qquad {\rm Vol} (S) \leq C \exp(-\epsilon^2/n), 
\end{align*}
where $C = \frac{1}{1 - {\rm Vol} (A_{k'})}$, showing that $q_{k'}$ is $(C, \epsilon^2/n, \delta)$-\red{localized}. If the classes were balanced, repeating the above argument for each class shows that $q_k$ is $(C, \epsilon^2/n, K\delta)$-\red{localized} for all $k \in \cY$ for $C_{\rm max} = \max_{k'} (1 / (1 - {\rm Vol}(A_{k'})))$. 
\end{proof}
%{\color{teal}
%\paragraph{Measure Concentration vs $(C, \epsilon, \delta)$-Concentration} 
%Recall that the definition of $(C, \epsilon, \delta)$-concentration requires the existence of $S \subseteq \cX$ such that $q(S) \geq 1 - \delta$ and $\vol{S} \leq C \exp(-\epsilon)$. We will now consider a few probability distributions and look at their measure concentration and data concentration.
%\begin{enumerate}
%\item Consider the uniform distribution on the unit cube $q_1 = {\rm Unif}([0, 1]^n)$. For any $S \subseteq [0, 1]^n$, we can solve $1 - \delta \leq q(S) = \vol{S} \leq \exp(-\epsilon)$ to obtain that 
%\begin{align*}
%q \text{ is } \left(1, \log \left( \frac{1}{1 - \delta} \right), \delta \right)-\text{concentrated for any }\delta \in [0, 1].
%\end{align*} 
%
%\item Consider the uniform distribution on the scaled unit cube $q_2 = {\rm Unif}([0,a]^n)$. For any $S \subseteq [0, a]^n$, we can solve $1 - \delta \leq q(S) = \frac{1}{a^n} \vol{S} \leq \frac{1}{a^n} \exp(-\epsilon)$ to obtain that 
%\begin{align*}
%q \text{ is } \left(1, \log \left( \frac{1}{1 - \delta} \right) + n \log \left(\frac{1}{a}\right), \delta \right)-\text{concentrated for any }\delta \in [0, 1].
%\end{align*} 
%From the above we can see that keeping $\delta, n$ fixed, $q_2$ becomes ``more concentrated'' as $a$ gets closer to $0$. Similarly, keeping $\delta, a < 1$ fixed, $q_2$ becomes ``more concentrated'' as the dimension $n$ increases. 
%\end{enumerate}
%}
\paragraph{Discussion on \cref{th:lzeroconc}}
A few comments are in order for the above result. %\ambar{Should we have a comment saying that the proof idea of finding the classification sets of $f$, and then performing suitable $\ell_p$ expansions is common to all of \cite{shafahi2018inevitable,dohmatob,dai2022fundamental,pal2023concentration} and others?}
\begin{enumerate}[leftmargin=0.5cm]
\item \cref{th:lzeroconc} demonstrates that whenever a $\ell_0$ robust classifier exists for a data distribution, this distribution must be \red{localized}. This could be instantiated for real data sets like ImageNet to obtain interesting observations about the underlying distribution. For instance, humans are robust to perturbation of a few pixels to any image in ImageNet. Then, \cref{th:lzeroconc} tells us that ImageNet is \red{localized}. Note, however, that the \red{localization} \emph{parameters} (\ie,  $C, \epsilon, \delta$ for the human classifier) are unknown. %\todo{JS: This is not really saying much} 

\item The \red{localization} parameters in \cref{th:lzeroconc} are different than the concentration parameters in \citet[Theorem 2.1]{pal2023concentration}. Specifically, \citet[Theorem 2.1]{pal2023concentration} shows that $(C, n \epsilon, \delta)$-concentration is a necessary condition for $\ell_2$-robustness under \cref{l0concdefn}, and we will now show that $(C, \epsilon^2 / n, \delta)$-\red{localization} is a necessary condition for $\ell_0$-robustness. This demonstrates that the existence of a classifier robust to  $\ell_0$ classifier implies a different kind of \red{localization} of the data distribution than robustness to $\ell_2$ perturbations. While \citet{pal2023concentration} assume that their data lies in a unit $\ell_2$ ball with adversarial perturbation strength $\epsilon \in [0, 1]$, we assume that our data lies in a unit $\ell_\infty$ ball and with perturbation strength $\epsilon \in \{0, 1, 2, \ldots, n\}$. As such a direct comparison of the parameters is not immediate as our work deals with objects very different from \citet{pal2023concentration}. 
%,  \ie, $\epsilon^2$ vs $\epsilon$, which shows that existence of classifiers robust to $\ell_0$ demonstrates a \emph{higher} concentration of the data distribution than robustness to $\ell_2$ perturbations. %This is in alignment with empirical observations in the literature where obtaining practically robust classifiers for $\ell_0$ seems to be harder than $\ell_2$. 

%{\color{blue}RV: I am very confused by the language here. Higher concentration is sort of presented as a good thing. I imagine fewer distributions are more highly concentrated. So I see it more as a tradeoff
%\color{teal} AP: Yes, mentions of ``higher'' are removed in favor of ``different''.}

\item \cref{th:lzeroconc} suggests that for obtaining $\ell_0$ robust classifiers, we should try to find and classify over the sets that the distribution \red{localizes} on. This is a significant departure from the existing literature on $\ell_0$-robust classifiers \cite{levine2020randomized,jia2022almost,hammoudeh2023feature}, and indeed, we will obtain a classifier in \cref{sec:boxclassifier} that respects such geometry.
\end{enumerate}

We have now demonstrated that \red{localization} is a necessary condition for the existence of a classifier robust to perturbations bounded in the $\ell_0$ distance, \ie, perturbations having a small support. Inspired by the investigations in \cite{pal2023concentration}, we will now consider whether this condition is also sufficient. 

\section{$d$-Strong \red{Localization} implies Existence of a $d$-Robust Classifier} \label{sec:sufficient}
%Following \cite{pal2023concentration}, 
\red{Localization} of the data distribution ensures that each class conditional concentrates on a small volume subset of $\cX$. However, as noted in \cite{pal2023concentration}, these subsets might intersect too much, in which case there might not exist a classifier with low standard risk, \ie, $R_{\ell_0}(f, 0)$. %even accurate classification might not be possible (let alone robust classification). 
Hence, one cannot expect \red{localization} to be sufficient for the existence of a classifier with low robust risk, \ie, $R_{\ell_0}(f, \epsilon)$ with $\epsilon > 0$. 
% condition for robustness. 
However, if these subsets were \emph{separated} enough, then one can expect to use them to build a robust classifier. Indeed, we will now formalize this intuition to obtain a condition stronger than \red{localization}, which will be shown to be sufficient for the existence of a robust classifier.

%Say we have a distance $d$, with respect to which we are measuring expansions of sets. We define strong concentration in $d$ as follows:
\begin{definition}[$d$-Strongly \red{Localized} Distributions, generalizing \cite{pal2023concentration}] \label{def:strongconc}A distribution $p$ is said to be $(\epsilon, \delta, \gamma)$-strongly-\red{localized} with respect to a distance $d$, if each class conditional distribution $q_k$ \red{localizes} over the set $S_k \subseteq \cX$ such that $q_k(S_k) \geq 1 - \delta$, and $q_k\left(\bigcup_{k' \neq k} S_{k'}^{+2\epsilon}\right) \leq \gamma$, where $S^{+\epsilon}$ denotes the $\epsilon$-expansion of the set $S$ in $d$, \ie, $S^{+\epsilon} = \{ x  \colon \exists \bar x \in S \text{ such that } d(x, \bar x) \leq \epsilon\}$.
\end{definition}

%{\color{blue}RV: where is the small volume part of the definition, with an exponent on $\epsilon$ that depends on the norm? \color{teal} This is implicit due to the separation condition.}

With the above definition, we will now obtain a generalization of \citet[Theorem 3.1]{pal2023concentration} to an arbitrary distance $d$: 
% \begin{restatable}[]{theorem}{lzerostrongconc}
% %\begin{theorem}
% If all the class conditionals $q_1, \ldots, q_k$ are $(\epsilon, \delta)$-concentrated over the sets $S_1, \ldots, S_k$, respectively, and $q_k\left( \cup_{k' \neq k} S^{+2\epsilon}_{k'} \right) \leq \gamma$, for all $k$, where expansions are w.r.t. the distance $d$, then there exists a classifier $f$ such that $R(f, \epsilon) \leq \delta + \gamma$. \label{th:strongconc}
% %\end{theorem}
% \end{restatable}

\begin{restatable}[]{theorem}{lzerostrongconc}
%\begin{theorem}
If $p$ is $(\epsilon, \delta, \gamma)$-strongly \red{localized} with respect to a distance $d$, then there exists a classifier $f$ such that $R_d(f, \epsilon) \leq \delta + \gamma$. \label{th:strongconc}
%\end{theorem}
\end{restatable}
\begin{proof} At a high level, we will construct a classifier $g$ that predicts the label $k$ over an $\epsilon$-expansion of the set $S_k$ on which the class conditional $q_k$ \red{localizes}. We will then ``shave off'' some regions from each $S_k$ to ensure $g$ is well defined. For the rest of the input space $\cX$ we will predict an arbitrary label, as we incur at most $\gamma$ in robust risk. Our construction of the robust classifier $f$ is same as that in \citet{pal2023concentration}, extended to general $d$. However, bounding the robust risk of $f$ needs technical innovations, since we are bounding the robust risk with respect to a general distance $d$, as opposed to the $\ell_2$ norm in \citet{pal2023concentration}.
 %The resultant $g$ incurs a robust risk of at most $\delta$ in $S_k^c$ and an additional $\gamma$ in the rest of the space.
%{\color{red}RV: I think a figure will go a long way. The intuition seems to simply be a collection of K balls being shrunk, yielding $C_k$, etc. No explanation for the 'otherwise' choice is given.}

For each $k \in \{1, 2, \ldots, K\}$, let $S_k$ be the set over which the conditional density $q_k$ is \red{localized}, i.e., $q_k(S_k) \leq 1 - \delta$. Define $S^{+\epsilon}$ to be the $\epsilon$-expansion of the set $S$, as $S^{+\epsilon} = \{x \colon \exists x' \in S, d(x, x') \leq \epsilon\}$. Define $C_k$ to be the $\epsilon$-expanded version of the \red{localized} region $S_k$ but removing the $\epsilon$-expanded version of all other regions $S_{k'}$, as
\begin{equation*}
C_k = \left( S_k^{+\epsilon} \setminus \cup_{k' \neq k} S^{+\epsilon}_{k'} \right) \cap \cX.
\end{equation*}
Similar to the construction in \cite{pal2023concentration}, we will use these regions to define the classifier $f \colon \cX \to \{1, 2, \ldots, K\}$ as
\begin{equation*}
f(x) = \begin{cases}
1, &\text{ if } x \in C_1 \\
2, &\text{ if } x \in C_2 \\
\vdots \\
K, &\text{ if } x \in C_K \\
1, &\text{ otherwise }
\end{cases}.
\end{equation*}
We will now show that $R_d(f, \epsilon) \leq \delta + \gamma$, which can be recalled to be %\ambartechnical{All balls are open here, to prevent problematic cases of $\bar x$ exactly on the boundary of a classification region}
\begin{align}
R_d(f, \epsilon) 
%&= \sum_k q_k(\{x \in \cX \colon \exists \bar x \in B_d(x, \epsilon) \cap \cX \text{ such that } f(\bar x) \neq k\}) p_Y(y = k) \nonumber \\
&= \sum_k q_k(U_k) p_Y(y = k), \label{eq:rdecomp}
\end{align}
where the $q_k$ mass in \eqref{eq:rdecomp} is over the set of all points $x \in \cX$ that admit an $\epsilon$-adversarial example for the class $k$, defined as
\begin{equation}
U_k = \{x \in \cX \colon \exists \bar x \in B_d(x, \epsilon) \cap \cX \text{ such that } f(\bar x) \neq k\}. \label{unsafekl0}
\end{equation}%
%{\color{red}This seems wrong that $U_k$ which is a set equals the sum of probabilities. Also, I would first define $U_k$ and then have expression for decomposition of $R$. Always number each equation.} {\color{blue} Whoops, yes that should go to previous line. Done.}
As we saw earlier in the proof of \cref{th:lzeroconc}, $U_k = \left(\cX \setminus C_k\right)^{+\epsilon} \cap \cX$. We will obtain an upper bound on $q_k(U_k)$, which will in turn give us an upper bound on $R_d(f, \epsilon)$. 

Let $A = S_k^{+\epsilon} \cap \cX$ and $B = \cup_{k' \neq k} S^{+\epsilon}_{k'}$. As $C_k = A \setminus B$, we have 
\begin{align*}
\cX \setminus C_k &= \cX \cap (A \cap B^c)^c \\
&= \cX \cap (A^c \cup B) \\
&= (\cX \cap A^c) \cup (\cX \cap B) \\
&= \left(\cX \cap \left( S_k^{+\epsilon} \right)^c \right) \cup \left( \cup_{k' \neq k} (\cX \cap S^{+\epsilon}_{k'}) \right).
\end{align*}
Then, we can expand $(\cX \setminus C_k)^{+\epsilon}$
\begin{equation*}
% &= \left(\cX \cap \left( S_k^{+\epsilon} \right)^c \right) \cup_{k' \neq k} (\cX \cap S^{+\epsilon}_{k'})
 \left(\cX \cap \left( S_k^{+\epsilon} \right)^c \right)^{+\epsilon} \cup \left( \cup_{k' \neq k} (\cX \cap S^{+\epsilon}_{k'})^{+\epsilon} \right),
\end{equation*}
from the property $(U \cup V)^{+\epsilon} = U^{+\epsilon} \cup V^{+\epsilon}$. Now, since all the mass of $q_k$ lies in $\cX$, \ie, $q_k(\cX) = 1$, we have $q_k(\cX \cap V) = q_k(V)$ for any set $V$. Applying this, we have
\begin{align*}
q_k(U_k) 
&= q_k (\cX \setminus C_k)^{+\epsilon} \\
&\leq q_k \left(\cX \cap \left( S_k^{+\epsilon} \right)^c \right)^{+\epsilon} + q_k\left( \cup_{k' \neq k} (\cX \cap S^{+\epsilon}_{k'})^{+\epsilon} \right) \\
&\leq q_k \left(\left( S_k^{+\epsilon} \right)^c \right)^{+\epsilon} + q_k\left( \cup_{k' \neq k} (S^{+\epsilon}_{k'})^{+\epsilon} \right). 
\end{align*}
%{\color{red}RV: why is this statement (that you can flip complement and expansion) true?} {\color{blue} This is Lemma A.1.}
Now applying \cref{lem:expandcontract} we have $\left( \left( S_k^{+\epsilon} \right)^c \right)^{+\epsilon} = \left(\left( S_k^{+\epsilon} \right)^{-\epsilon} \right)^c$. Again from \cref{lem:expandcontract} we know that $(V^{+\epsilon})^{-\epsilon} \supseteq V$ for any set $V$. Hence, we have $\left(\left( S_k^{+\epsilon} \right)^{-\epsilon} \right)^c \subseteq S_k^c$. Continuing, 
\begin{align*}
q_k(U_k) 
&\leq q_k(S_k^c) + q_k\left( \cup_{k' \neq k} S^{+2\epsilon}_{k'} \right) \\
&\leq \delta + \gamma,
\end{align*}
Finally, as $\sum_k p_Y(y = k) = 1$, from \eqref{unsafekl0} we have $R_d(f, \epsilon) \leq \delta + \gamma$.
\end{proof}

We note that \citep[Theorem 3.2]{pal2023concentration} follows as a direct corollary of our result \cref{th:strongconc} by taking $d$ to be the $\ell_2$ distance. 

\paragraph{Implications for Existing Impossibility Results} In our setting, \cite{shafahi2018inevitable} prove that for any classifier $f \colon \cX \to \{1, 2, \ldots, K\}$ for any class $k$ with $P(Y = k) \leq 1/2$, any point $x \sim q_k$ is either mis-classified, or admits an  $\epsilon$-adversarial example with probability at least 
\begin{equation}
    1 - \beta_{q_k} \exp \left( - \epsilon^2 / n \right), \label{goldsteinbound}
\end{equation}
where $\beta_{q_k} = 2 \sup_x q_k(x)$ depends on the class conditional $q_k$. When $q_k$ is \red{localized}, $\beta_{q_k}$ can grow faster than $\exp \left( - \epsilon^2 / n \right)$, making the lower bound vacuous. This implies that \emph{for \red{localized} data-distributions there is no impossibility, and there is a wide class of high-dimensional classification problems for which robust classifiers exist}. We now provide a concrete example.
\begin{example}
\label{l0example}
Let us consider a problem with $2$ classes defined by the distribution $p$ such that $P(Y=0) = P(Y=1) = 1/2$, the class conditional $q_1 = P(X | Y = 1) = {\rm Unif}(B_{\ell_\infty}(\mathbf{1}, \epsilon))$, and similarly $q_2 = P(X | Y = 2) = {\rm Unif}(B_{\ell_\infty}(-\mathbf{1}, \epsilon))$. For this distribution, $\beta_{q_1} = \beta_{q_2} = \exp(n)$, and the lower bound \eqref{goldsteinbound} becomes vacuous for $\epsilon \leq \sqrt{n}$ as 
\begin{equation*}
    1 - \beta_{q_k} \exp \left(-\epsilon^2 / n \right) = 1 - 2 \exp(-\epsilon^2/n + n) \leq 0.
\end{equation*}
\end{example}
%{\color{red}I think this point needs to be elaborated further. I got it, but I think a reviewer may miss it, and I think it is a critical point. In fact, I think what you are trying to say, which is that impossibility results no longer hold, is said implicitly and technically, but not in lay men words.}
Even though \cref{l0example} is quite simple, the construction of small $\ell_\infty$ balls in the input space containing most of the mass of the distribution is quite general, and depicts a wide class of data-distributions where existing impossibility results are vacuous. We will now demonstrate that these general theoretical ideas lead to practical $\ell_0$ robust classifiers.

\section{$\ell_0$-Adversarially Robust Classification via the \textsc{Box-NN} classifier} \label{sec:boxclassifier}
%{\color{red}I feel there is not enough motivation for the section. The section then reads as this example, this box, this other box, this final box, theorem about the final box. The reader doesn't know where we are going.}
In this section, our aim will be to derive a $\ell_0$-robust classifier by utilizing the geometry exposed by \cref{th:strongconc}. To this end, we will first investigate how a robust classifier looks like for a simple 2-class problem in 3-dimensions. This will motivate a general form of a classifier whose decision regions are axis-aligned cuboids, or boxes. Finally, we will generalize this classifier to obtain a $\ell_0$-robust classifier and derive  corresponding $\ell_0$ certificates. 
%take a closer look at the geometry of a $\ell_0$-robust classifier, as exposed by \cref{th:strongconc}, and derive a robust classifier based on this geometry. 

\subsection{Development and Robustness Certification}
Consider $n = 3$, and say there are two classes, \textsc{cat} and \textsc{dog}, defining conditional distributions $q_1$ and $q_2$, strongly \red{localized} over $S_1$ and $S_2$ respectively, such that $q_1(S_2^{+1}) = 0$ and $q_2(S_1^{+1}) = 0$. In such a situation, \cref{th:strongconc} (invoked with $\epsilon = 1$) constructs a robust classifier $f_A$ as the following:
\begin{equation*}
f_A(x) = \begin{cases}
\text{dog}, &\text{ if } x \in S_1^{+1} \\
\text{cat}, &\text{ if } x \in S_2^{+1} \\
\text{cat}, &\text{otherwise}.
\end{cases}.
\end{equation*}
%{\color{red}Why it is not $S_1^+\setminus S_2^+$?} {\color{blue} They don't intersect for this illustration.}

However, in practice, $S_1, S_2$ might be very complex, and hence $f_A$ might be computationally hard to evaluate. For instance, \cref{fig:boxnn1} shows an illustration where these sets (shaded green and orange) have complicated shapes.
\begin{figure}[h!]
\centering
\includegraphics[width=0.3\textwidth]{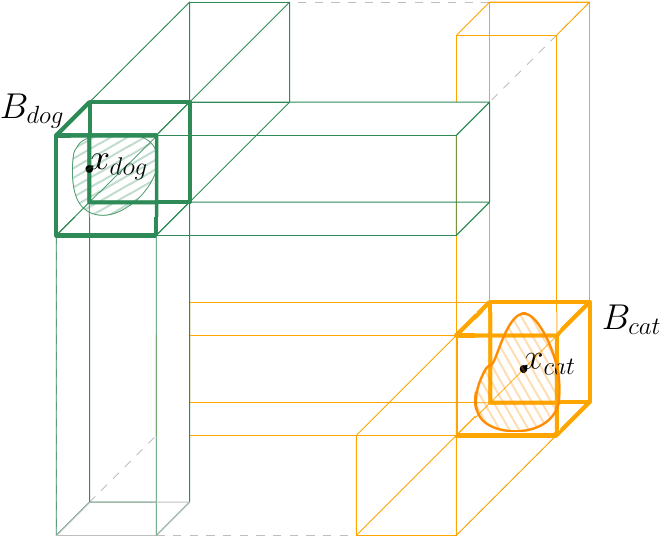}
\caption{$S_1$ is the green shaded region around $x_{dog}$, where the class dog is \red{localized}, and $S_2$ is the orange shaded region around $x_{cat}$, where the class cat is \red{localized}.}
\label{fig:boxnn1}
\end{figure}

From \cref{fig:boxnn1}, we see that the classifier $f_A$ is robust to $1$-pixel perturbations whenever $x \in S_1$ or $x \in S_2$, as \cref{th:strongconc} predicts. More importantly, we see that a perturbation of a single pixel of any $x_{cat} \in S_2$ lies within the union of the orange cuboids. In other words, $\{x' \in [0, 1]^3 \colon \|x - x\|_0 \leq 1, x \in S_1\} = S_1^{+1} \subseteq \textsc{Orange}$, and similarly for the dog class. Furthermore, we see that the intersection of these orange cuboids is given by the cube $B_{cat}$. We can see that for any $x \in B_{cat}$, no single-pixel perturbation $v$ can take $x + v$ outside the orange region $\textsc{Orange}$, and similarly for the dog class. However, $B_{cat}, B_{dog}$ are very efficiently described, they are simply axis-aligned polyhedra enclosing $S_2$ and $S_1$ respectively. This motivates our modified classifier $f_B$, 
\begin{equation*}
f_B(x) = \begin{cases}
\text{dog}, &\text{ if } x \in B_{dog}^{+1} \\
\text{cat}, &\text{ if } x \in B_{cat}^{+1} \\
\text{cat}, &\text{otherwise}.
\end{cases}.
\end{equation*}

While $f_B$ is efficient to describe, it ignores a large portion of the input region outside the green and the orange cuboids, \ie, $\cX \setminus B_{dog}^{+1} \cup B_{cat}^{+1}$, by making the constant prediction \text{cat} in this region. We can further extend $f_B$ to attempt to correctly classify those regions as well, by computing $\ell_0$ distances to our boxes $B_{cat}, B_{dog}$, as
\begin{equation*}
%\textsc{Box-NN}
f_C(x) = \argmin_{y \in \{\text{cat}, \text{dog}\}} {\rm dist}(x, B_y),
\end{equation*}
where 
\begin{equation}
{\rm dist}(x, S) = \min_{v} \|v\|_0 \text{ sub. to } x + v \in S \label{eq:dist}
\end{equation}
gives the minimum number of pixel changes needed to get from $x$ to $S$. While solving \eqref{eq:dist} is computationally hard for general $S$, the following lemma shows that for our axis-aligned boxes $B$, \eqref{eq:dist} can be computed efficiently, in closed form. The proofs of all our results can be found in \cref{app:proofs}. 

\begin{restatable}[$\ell_0$ distance to axis-aligned boxes]{lemma}{lzerodist}
%\begin{lemma}[$\ell_0$ distance to axis-aligned boxes] 
\label{lem:l0dist} For an axis aligned box $B(a, b)$ specified as $B(a, b) = \{x \colon a \leq x \leq b\}$, where $a, b, x \in \R^n$, and all inequalities are element-wise, we have 
\begin{equation*}
{\rm dist}(x, B(a, b)) = \sum_{i = 1}^n \1 \left(x_i \not \in [a_i, b_i]\right), 
\end{equation*}
which can be computed in $O(n)$ operations. 
%\end{lemma}
\end{restatable}

For real data distributions, however, having a single box per class would be overly simplistic and not provide good accuracy. Thus, we generalize $f_C$ to our \textsc{Box-NN} classifier operating on boxes $\cB = \{B_1, B_2, \ldots, B_M\}$, such that we have an label $y_m \in \{1, 2, \ldots, K\}$ associated with each $B_m$. Our \textsc{Box-NN} classifier is then defined as
\begin{equation*}
\textsc{Box-NN}(x, \cB) = y_{m^\star}, \text{ where } m^\star = \argmin_{m} {\rm dist}(x, B_m).
\end{equation*}

Note that, so far, we have not described how these boxes $\cB$ are learned from data. This will be the subject of \cref{sec:empirics} and onward. %
%{\color{red}Where do the boxes come from?}
We can now obtain a 
%\emph{tight} 
$\ell_0$ robustness certificate for $\textsc{Box-NN}$ via the following Theorem.

%{\color{blue} AP: Check overloading.}
\begin{restatable}[Robustness Certificate for \textsc{Box-NN}]{theorem}{thcert}
%\begin{theorem}[Robustness Certificate for \textsc{Box-NN}] 
\label{th:cert} Given a set of boxes $\cB$ and their associated labels $\{y_m\}_{m=1}^M$, define 
\begin{equation*}
m^\star = \argmin_{m} {\rm dist}(x, B_m), \quad d_1 = {\rm dist}(x, B_{m^\star}),
\end{equation*}
and 
\begin{equation*}
d_2 = \min_{m \colon y_m \neq y_{m^*}} {\rm dist}(x, B_m).
\end{equation*}
%{\color{red}I don't like $m_in$. It is a distance, not the index of the box. Also, in the result, the difference $m_{out}-m_{in}$ is really a margin and that is not even mentioned.}
Then, with ${\rm margin}(x) \overset{\rm def}{=} d_2 - d_1$, we have $
\textsc{Box-NN}(x, \cB) = \textsc{Box-NN}(x', \cB)$ whenever $\|x' - x\|_0 < {\rm margin}(x) / 2$.
%Further, the certificate is tight, \ie, there exists an $x'$ such that 
%$\|x' - x\|_0 \geq {\rm margin}(x) / 2$ and $\textsc{Box-NN}(x, \cB) \neq \textsc{Box-NN}(x', \cB)$.
%\end{theorem}
\end{restatable}
%
%{\color{teal}RV: My main complain till here is that it reads as each class is a box, the boxes are sufficiently separated, I compute distance to box in my class, distance to boxes for other classes to define margin, and I am robust to perturbations withing 1/2 of margin. Like nearest box is a 1-Lypchitz classifier. Overall, the mathematical writing does not elucidates this intuition of what's going on behind. Once a reviewer figures it out, the reviewer may think it is too cartoonish. Too simple. So the wriitng needs to provide the high level and argue why it is not too cartooshish. Argue for the challenges.}
%
\paragraph{Key Intuition} Our robust classifier \textsc{Box-NN} is essentially a generalization of the nearest-neighbor classifier to a nearest-box classifier, specifically suited to $\ell_0$ metrics. This simple form turns out to be the right choice, in the sense of the theoretical motivation of our previous section, for defending against sparse perturbations. As we will shortly see, \textsc{Box-NN} also empirically produces better certificates than prior work in several regimes.

Having developed the geometric intuition and the theoretical robustness guarantees for \textsc{Box-NN}, we will now describe how we learn our classifier from data, and the associated challenges.

%\textbf{\color{green} Rene read till here} {\color{blue} AP: Text above has been updated since last read}
\subsection{Learning \textsc{Box-NN} from Data} \label{sec:empirics}
In this section, we are concerned with learning boxes $\{B_m\}$ and their associated labels $\{y_m\}$, such that \textsc{Box-NN} obtains a high accuracy under sparse adversarial perturbations. For the rest of this section, we will refer to the classifier $\textsc{Box-NN}$ as $f_\theta$, with the learnable  parameters $\theta = \{a_k, b_k, y_k\}_{k=1}^M$ following the notation in \cref{lem:l0dist}. 

The quantity we are interested in maximizing is the robust accuracy, defined as $1 - R_{\ell_0}(f_\theta, \epsilon)$ following our notation in \cref{sec:necessary}. As we do not have access to the data distribution, we will instead be concerned with maximizing the empirical robust accuracy ${\rm RobustAcc}(f_\theta, \epsilon) $ defined over a set of samples $\{x_i, y_i\}_{i=1}^N$ given by
\begin{equation}
%&=  1 - \frac{1}{N} \sum_{i = 1}^N \1 \left[ \exists x' \colon \|x' - x_i\|_0 \leq \epsilon \text{ s.t. } f_\theta(x') \neq y_i \right] \nonumber \\
%&=  
\frac{1}{N} \sum_{i = 1}^N \1 \left[ \forall x' \colon \|x' - x_i\|_0 \leq \epsilon, \  f_\theta(x') = y_i \right]. \label{eq:robacc}
\end{equation}

%\begin{equation*}
%{\rm RobustAcc}(\textsc{Box-NN}, \epsilon) =  1 - R_{\ell_0}(\textsc{Box-NN}, \epsilon).\label{eq:robacc}
%\end{equation*}
%{\color{teal}RV: The current writing does not read well from 'called onwards'. I would first say we want to achieve high robust accuracy, then say, i.e..} 
The objective in \eqref{eq:robacc} is a complicated object, and direct maximization w.r.t. $\theta$ is challenging. In the following, we will first lower bound \eqref{eq:robacc} and then use several optimization tricks to efficiently maximize this lower bound.% using gradient information.

{\color{teal}
}

Recall from \cref{th:cert} that $f_\theta(x) = f_\theta(x')$ for all $x'$ satisfying $\|x - x'\|_0 \leq C_\theta(x) \overset{\rm def}{=} {\rm margin}(x)/2$, where $C_\theta$ is a pointwise certificate (at $x$)  of robustness for $f_\theta$. With this, we have the certified accuracy lower bound ${\rm RobustAcc}(f_\theta, \epsilon) \geq {\rm CertAcc}(f_\theta, \epsilon)$ defined as
\begin{equation}
{\rm CertAcc}(f_\theta, \epsilon) \overset{\rm def}{=} \frac{1}{N} \sum_{i = 1}^N \1[f_\theta(x_i) = y_i] \cdot \1[C_\theta(x_i) \geq \epsilon]. \label{certacc}
\end{equation}

%\paragraph{Gradient-based learning} 
We will take a gradient based optimization approach to maximize \eqref{certacc} over $\theta$. However, since the gradients of $\1[\cdot]$ are zero almost everywhere (and discontinuous otherwise), we will progressively relax the indicators in \eqref{certacc}. To this end, we maximize the integral of ${\rm CertAcc}(f_\theta, \epsilon)$ over all $\epsilon \geq 0$ instead of treating it point-wise\footnote{\ie, $\int_{\epsilon \geq 0} \1[\epsilon \leq \alpha] d \epsilon = \alpha$}, leading to the objective
%{\color{red}I am not sure why this is being done this way. I am not sure why 'summing over $\epsilon$' transforms (8) into (9). The rationale behind teh direction being taken is not clear.} 
\begin{equation}
    L_1(\theta) = \frac{1}{N} \sum_{i = 1}^N \1[f_\theta(x_i) = y_i] \cdot C_\theta(x_i). \label{eq:sparse}
\end{equation}

% \paragraph{Relaxing $\min$}
On the other hand, recall from \cref{th:cert} that the ${\rm margin}$ involves the $\min$ function,
\begin{equation*}
{\rm margin}(x) = \min_m {\rm dist}(x, B_m) - \min_{m \colon y_m \neq y_{m^\star}} {\rm dist}(x, B_m).
\end{equation*}
The gradient of $\min$ w.r.t. its input $(c_1, \ldots, c_M)$ is extremely sparse\footnote{$\nabla_c \min (c_1, c_2, \ldots, c_m) = (0, \ldots, 0, 1, 0, \ldots, 0) = e_{j^\star}$, where $e_j$ is the $j^{\rm th}$ standard basis vector, and $j^\star = \argmin_j c_j$.}%
%{\color{red}RV: why?}
, and hence a very small number of parameters $\theta_i$ are updated at each step of gradient descent using gradients of \eqref{eq:sparse}. As a result, optimization is extremely slow. We remedy this by using a soft approximation to $\min$ which has dense gradients, 
\begin{equation}
    %\min\{c_1, \ldots, c_M\} \approx 
    {\rm min}_\tau \{c_1, \ldots, c_M\} \overset{\rm def}{=} \sum_{m = 1}^M c_m \frac{\exp (-\tau c_m)}{\sum_j \exp (-\tau c_j)} \label{eq:softmin},
\end{equation}
where $\tau$ is a parameter that approximately controls the sparsity of the gradients. The function $\min_\tau$ is equal to $\min$ in the limit $\tau \to \infty$, and reduces to the average when $\tau = 0$. %The approximation in \eqref{eq:softmin} is also known as \texttt{softmin}. 
This step is crucial for the performance of our method. %A few other possible approximations are presented in the Appendix. %Note that a related line of research tries to estimate gradients through the $\argmin$ function (\eg, straight-through estimators \cite{??}, gumbel softmax \cite{??}, reinforce \cite{??}). 
%{\color{red}RV: there is a ton of straight through approximators, there was even an entire paper at neurips comparing them (first session on theory I think, Aditya knows). I think the presentation here could be improved with more context and references. \url{https://fabianfuchsml.github.io/gumbel} {\color{blue} AP: I think these are two different things. The gumbel-max is a way to sample efficiently from a categorical distribution, and the gumbel-softmax is a way to have a sample $y \in \Delta^{k-1}$ (in the interior of the simplex) from a categorical distribution specified by parameters $\pi \in \R^k$ such that $\partial y / \partial \pi$ is well defined. The straight-through gumbel-softmax is a trick to estimate $\partial z/\partial \pi$ when $z$ is a discretized version of $y$, \ie, $z = {\rm OneHot}(y)$. We are doing none of these things here, no discrete sampling is going on here. Here, we are simply interested in $\nabla_c \min(c_1, \ldots, c_m)$, and are making the $\min_\tau$ approximation so as to not have to deal with very sparse gradients.}}

% \paragraph{Clipping large certificates} 
Furthermore, we find that for many data points $x_i$, a small number of boxes $m$ contribute a lot to the final loss due to large distances ${\rm dist}(x_i, B_m)$. As a result, learning is slow for parameters corresponding to the remaining boxes. To prevent such imbalance, we clip the certificates to $50$. 
%{\color{red}RV:seems like a hack. Why not motivate robust distances to begin with? \color{blue} AP: Good question, I haven't thought about it much, but having robust distances makes the certification much harder I think. Empirically, a simple clipping works.} 
With these approximations, we obtain 
\begin{equation*}
L_2(\theta) = \frac{1}{N} \sum_{i = 1}^N \1[f_\theta(x_i) = y_i] \cdot \tilde{C}_\theta(x_i),
\end{equation*}
where $\tilde{C}_\theta(x)$ is defined as
\begin{equation}
%\tilde{C}_\theta(x) = 
\min \left(\underset{m}{ {\rm min}_\tau } \  {\rm dist}(x, B_m) - \underset{m \colon y_m \neq y_{m^\star}}{ {\rm min}_\tau } {\rm dist}(x, B_m), 50 \right). \label{eq:indicator}
\end{equation}
%{\color{red}RV: why $y_m^*$ is needed? Isn't $y_m^*$ the result of the first min, or is it the truth. I'm confused as to why softmin - softmin makes sense. Wouldn't one use softmargin instead? \color{blue} AP: It is the result of the first min, as defined in Theorem 4.2. I don't follow what you mean by softmargin here.}

\paragraph{Relaxing Indicator Functions} Now observe that $L_2$ is still a function of indicator functions, due to the ${\rm dist}$ function in \eqref{eq:indicator}, which was derived in \cref{lem:l0dist} to be ${\rm dist}(x, B(a, b)) = \sum_{i = 1}^n \1\left(x_i \not \in [a_i, b_i] \right)$. Again, as the gradients of $\1[\cdot]$ are zero almost everywhere, we perform a conical approximation to $\1\left(x_i \not \in [a_i, b_i] \right)$ which has non-zero gradients:
\begin{equation*}
     {\rm conical}(x, a_i, b_i) \overset{\rm def}{=} \max(a_i - x, 0) + \max(x - b_i, 0).
\end{equation*}
%{\color{red}I'm sorry, gradients not being useful and some gradients being better than others is far too colloquial for me. The writing needs to be more rigorous.}

Finally, we replace the indicator $\1[f_\theta(x_i) = y_i]$ in $L_2$ by $s_i$, where $s_i = +1$ if $f(x_i) = y_i$, and $s_i = -1$ otherwise, to have the misclassified data-points contribute to the loss. These modifications lead to our final objective $L(\theta)$. %\todo{AP: Writing full form of L is notationally ugly}%, where the approximations can be summarized as 
%\begin{equation*}
%{\rm CertAcc} \xrightarrow{\text{sum } \epsilon} L_1 \xrightarrow{{\rm min}_\tau, \text{ clip} } L_2 \xrightarrow{\text{relax } \1[\cdot]} L
%\end{equation*}
%\begin{equation*}
%\frac{1}{N} \sum_{i = 1}^N s_i \cdot \tilde{C}_\theta(x_i),
%\end{equation*}

\paragraph{Improving Initialization} We initialize $\theta$ by using a set of boxes defined from the data. This is done by first drawing a subset $T$ of size $M$ uniformly at random from the training data-points, and then initializing $\theta$ with axis-aligned boxes centered at these data-points, as $\{(B(x - 0.1, x + 0.1), y) \colon (x, y) \in T\}$, where $+$ denotes vector-scalar addition. Having described all the tricks used for optimizing \textsc{Box-NN}, we now proceed to performing an empirical evaluation.

\section{Empirical Evaluation} 
In this section, we will briefly describe existing methods for probabilistic $\ell_0$ certification, \citep{levine2020robustness} and \citep{jia2022almost} as well as deterministic $\ell_0$ certification \citep{hammoudeh2023feature}, and then empirically compare our (deterministic) $\ell_0$ certified defense $\textsc{Box-NN}$ to these approaches.
%In this section, we will compare our (deterministic) $\ell_0$ certified defense to existing methods for probabilistic $\ell_0$ certification, \citep{levine2020robustness} and \citep{jia2022almost} as well as a very recent method for deterministic $\ell_0$ certification \citep{hammoudeh2023feature}. 
{\color{blue} %We will state the general framework for comparison, and then briefly describe these methods in terms of this framework.
%\todo{The second sentence does not read well in context with the first. Will the section have three parts (a) Comparison framework (b) Methods (c) Results? If so, say that clearly and have subsections with those three titles. But the issue is that you open with (c) and then say the section is about (a) and (b)}
%MNIST and the FashionMNIST datasets. We will now briefly describe each robust classifier and its associated certificate.
}

\citet{levine2020robustness} and \citet{jia2022almost} extend the technique of randomized smoothing \citep{cohen2019certified} to randomized ablation (RA), where given any classifier $f$ (\eg, a neural network), they produce a smoothed classifier $g$ by zeroing out $k$ pixels uniformly at random:
\begin{equation}
g_{\rm RA}(x) = \argmax_k \bbP_{v \sim {\rm Unif}(S)} \left( f(x \odot v) = k \right), \label{eq:ablation}
\end{equation}
where $S = \{ v \in \{0, 1\}^n \colon \|v\|_0 = n - \rho\}$ is the discrete set of all binary vectors of length $n$ having exactly $\rho$ zeros, and $\odot$ denotes the Hadamard product. For this construction in \eqref{eq:ablation}, a counting argument leads to the robustness certificate %$C_{\rm RA1}(x, \epsilon)$
in \cite{levine2020robustness}, which we compare to in \cref{fig:comprand}. 
%which asserts $g(x) = g(x')$ for all $x'$ such that $\|x' - x\|_0 \leq \epsilon$. 
A more complicated analysis based on the Neyman-Pearson lemma leads to a tighter certificate %$C_{\rm RA2}(x, \epsilon)$ 
in \cite{jia2022almost}, which is also included in our comparison in \cref{fig:comprand2} (left).%\footnote{We compare to an implementation of \cite{jia2022almost} reported in \citet[Table 27]{hammoudeh2023feature}, given that no public implementation of the method is available.}. 
Both these certificates are randomized, \ie, they hold with a confidence $1 - \alpha$, where $\alpha, \rho$ are hyper-parameters that trade-off benign accuracy to robustness, and can be chosen empirically. According to standard practice, we fix $\alpha = 0.05$ and produce plots for varying $\rho$. The interested reader can refer to  \citep{levine2020robustness, jia2022almost} for a detailed description of these certification procedures.

\begin{figure}[h!]
\centering
\includegraphics[width=0.4\textwidth]{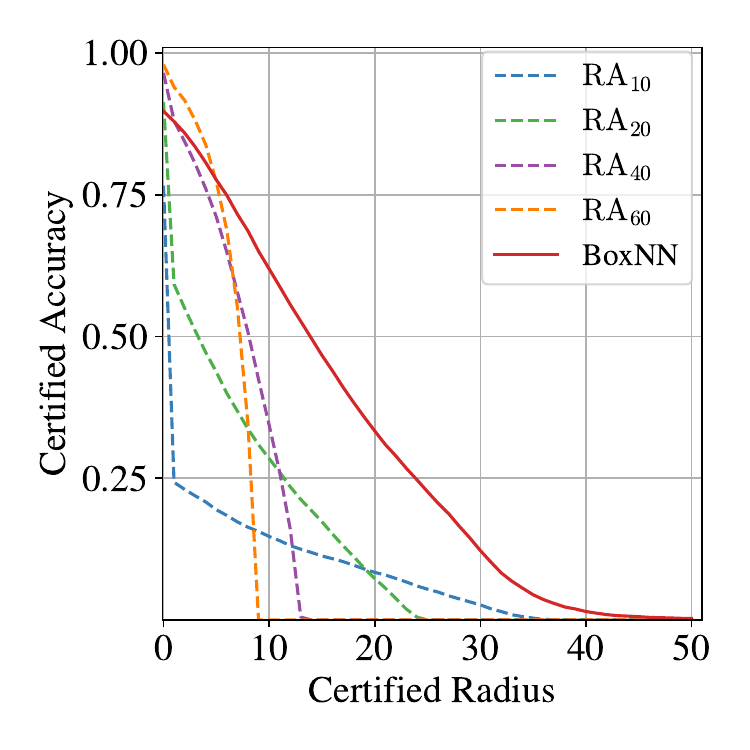} \hspace{-5mm}
\includegraphics[width=0.4\textwidth]{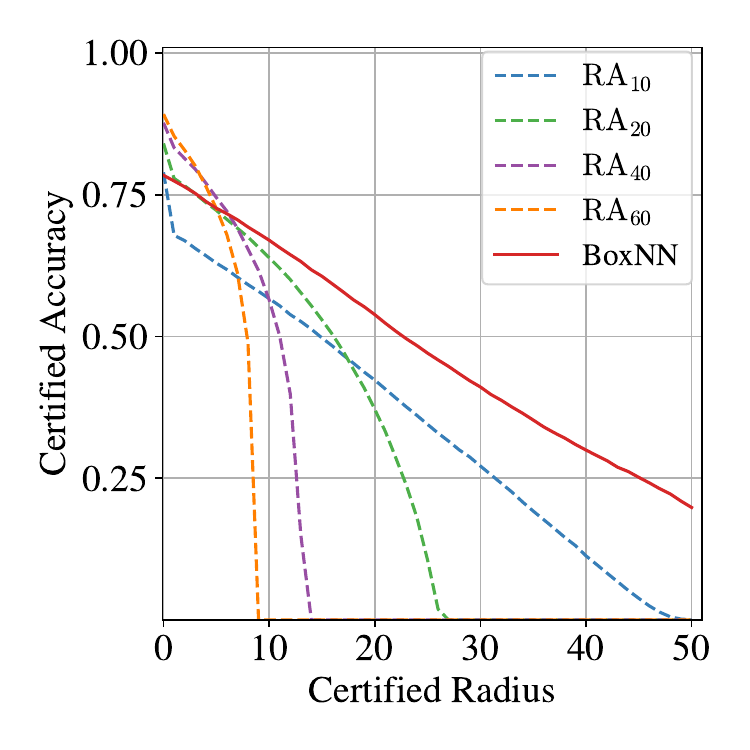}
\caption{Comparison of Randomized Ablation \citep{levine2020robustness} to our method \textsc{Box-NN} on the MNIST (left) and FashionMNIST (right) datasets. In each figure, the dotted lines correspond to different hyperparameter settings $\rho$. Details in text.}
\label{fig:comprand}
\end{figure}

\begin{figure}[h]
\centering
\includegraphics[width=0.4\textwidth]{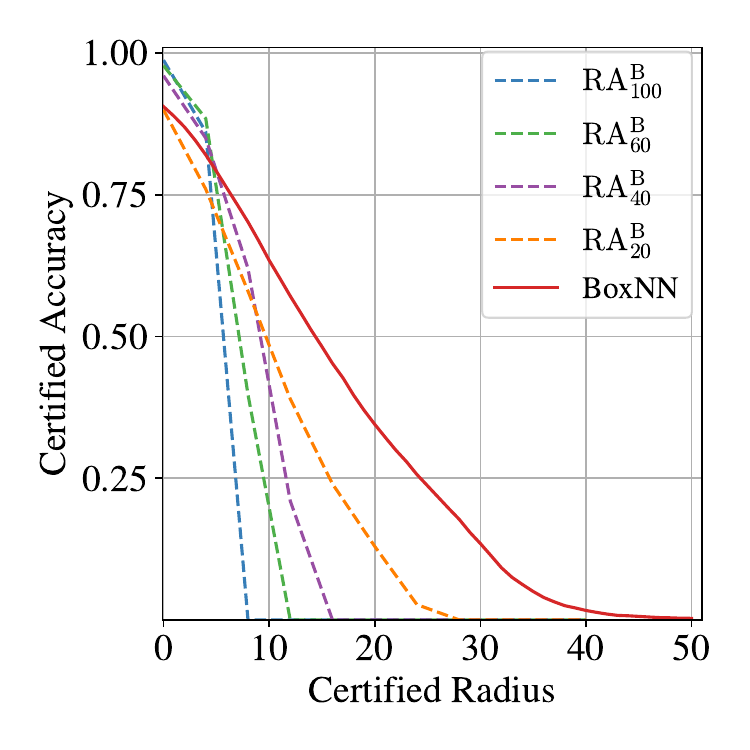} \hspace{-5mm}
\includegraphics[width=0.4\textwidth]{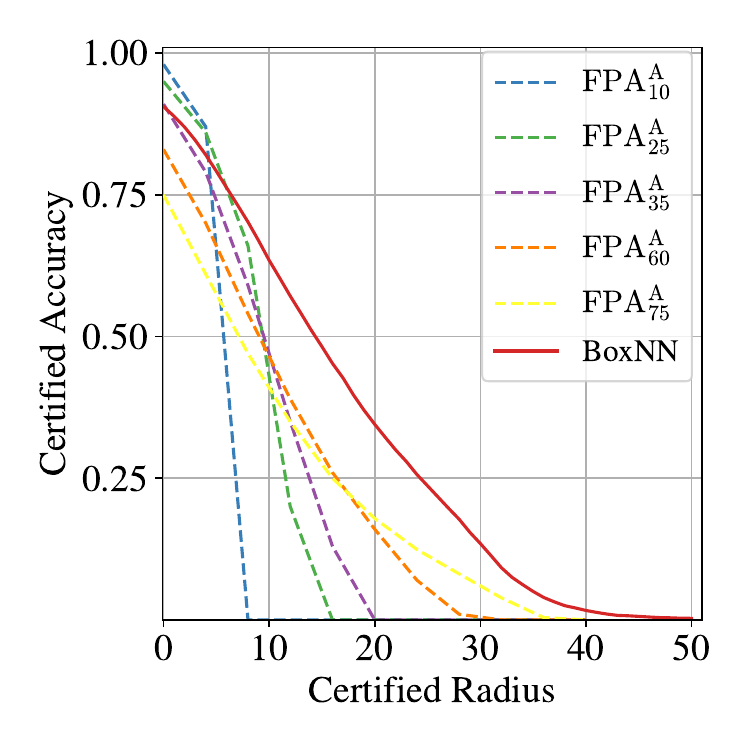}
\caption{Comparison of \cite{jia2022almost} (left) and \cite{hammoudeh2023feature} (right) to our method \textsc{Box-NN} on the MNIST dataset. The dotted lines correspond to different settings for the hyperparameter $\rho$. Details are mentioned in text.}
\label{fig:comprand2}
\end{figure}

More recently, given any classifier $f$, \citet{hammoudeh2023feature} produce a determinstic $\ell_0$ certified classifier $g$ by partitioning the set of pixels $\{1, 2, \ldots, n\}$ into disjoint partitions $\cS$, and then producing the majority prediction of $f$ over $\cS$:
\begin{equation}
g_{\rm FPA}(x) = {\rm Majority}\{f(x_S)\}_{S \in \cS}, \label{eq:ensemble}
\end{equation}
% \todo{RV: Why majority vs probabilistic?}
%{\color{red}RV: in what way is this different? That is, if the sets $S$ are selected uniformly at random, then $x_S$ is the same as $x\odot v$ with $v = 1_S$, and the argmax is the same as majority vote. Also, I do not see results for EN in Fig. 2. Then, Fig. 3 people may not know which one is faded blue.}
where $f(x_S)$ is defined as the prediction of $f$ obtained after zeroing out the pixels in $x$ not in $S$. \citet{hammoudeh2023feature} then produce a certificate by counting the difference in the votes of the majority label to the runner-up label in \eqref{eq:ensemble}. In \cref{fig:comprand2} (right), we compare to the best performing strategy for constructing $\cS$ in \citep{hammoudeh2023feature} named ``strided'', where equally spaced pixels are selected for each partition, \ie, $\cS = \{p \colon p \equiv t - 1 \text{ mod } \rho \}_{t = 0}^{\rho - 1}$. Here $\rho$ is a hyper-parameter as earlier, and we vary $\rho$ to produce the plots in \cref{fig:comprand2}\footnote{We use the results reported in \citet[Table 27]{hammoudeh2023feature} given that no public implementation of the method is available, to the best of our knowledge.}. Note that \citep{hammoudeh2023feature} also obtain an improved certificate by using an aggregation more complicated than the majority vote, which we compare to in \cref{fig:compdet}. The interested reader can refer to \cref{app:addlcomp} and \citep{hammoudeh2023feature} for more details.

\textbf{Results}\  Recall from \cref{sec:empirics} Eq. \eqref{certacc} that the certified accuracy of a classifier $g$ against $\epsilon$-bounded adversarial perturbations, ${\rm CertAcc}(g, \epsilon)$, can be obtained given a point-wise certificate $C$ for $g$. For each of the methods described so far, we plot ${\rm CertAcc}$ against $\epsilon$ using the corresponding robust classifier $g$ and the certificate $C$ over samples from the test set of the datasets mentioned. 

 \begin{table}[t]
\caption{Comparison of the median certified radius $\bar r$ obtained by our \textsc{Box-NN} to the best hyperparameter settings for prior work.}
\label{tab:median}
\vskip 0.03in
\begin{center}
\begin{small}
%\begin{sc}
\begin{tabular}{cccc}
\toprule
\textsc{Dataset} & \textsc{Method} & $\bar r$ \\
\midrule
\multirow{5}{*}{MNIST}  & \textsc{Box-NN} & $\mathbf{13}$\\
& RA \cite{levine2020robustness} & $8$ \\
& ${\rm RA}^{\rm B}$ \cite{jia2022almost} & $10$  \\
& \makecell{${\rm FPA}^{\rm A}$ \\ \cite{hammoudeh2023feature}} & $9$  \\
& \makecell{${\rm FPA}^{\rm B}$ \\ \cite{hammoudeh2023feature}} & $12$ \\
\midrule
\multirow{2}{*}{\makecell{FMNIST}}  & \textsc{Box-NN} & $\mathbf{22}$ \\
& RA \cite{levine2020robustness} & 16 \\
\bottomrule
\end{tabular}
%\end{sc}
\end{small}
\end{center}
\vskip -0.2in
\end{table}
 
%\begin{table}[t]
%\caption{Comparing the median certified radius (MCR) obtained by our \textsc{Box-NN} to the best hyperparameter settings for prior work.}
%\label{tab:median}
%\vskip 0.15in
%\begin{center}
%\begin{small}
%%\begin{sc}
%\begin{tabular}{p{1cm}ccc}
%\toprule
%\textsc{Dataset} & \textsc{Method} & $\rho$ & \textsc{MCR} \\
%\midrule
%\multirow{5}{*}{MNIST}  & \textsc{Box-NN} &  &  \\
%& RA \cite{levine2020robustness} & 40 & 8 \\
%& ${\rm RA}^{\rm B}$ \cite{jia2022almost} &  & 10  \\
%& \makecell{${\rm FPA}^{\rm A}$ \\ \cite{hammoudeh2023feature}} &  & 11  \\
%& \makecell{${\rm FPA}^{\rm B}$ \\ \cite{hammoudeh2023feature}} &  & 13 \\
%\midrule
%\multirow{2}{*}{\makecell{FMNIST}}  & \textsc{Box-NN} &  &  \\
%& RA \cite{levine2020robustness} &  &  \\
%\bottomrule
%\end{tabular}
%%\end{sc}
%\end{small}
%\end{center}
%\vskip -0.4in
%\end{table}

 A commonly used metric for comparing certified accuracy curves adopted in the literature \citep{levine2020robustness,jia2022almost,hammoudeh2023feature} is the median certified radius, which is the largest perturbation strength under which a classifier is certified to have atleast $50\%$ robust accuracy. As can be seen in \cref{tab:median}, our method \textsc{Box-NN} \emph{outperforms all existing methods under all hyperparameter settings on this metric}. 

The median certified radius captures a small slice of the full certified accuracy curve, which provides a complete picture. Observe that the dotted curves in \cref{fig:comprand,fig:comprand2} remain lower than our red curve except at small attack strengths. This shows that \textsc{Box-NN} is able to produce better certificates at most radii, and trades-off robustness at higher radii for benign accuracy at small radii. Without any dedicated hyper-parameter tuning, \textsc{Box-NN} dominates any single dotted curve for a large range of attack strengths, demonstrating that certified defenses closely utilizing properties of the data-distribution can outperform complicated ensembling-based defenses which ignore properties of the data. %{\color{blue} AP: We have the same trends for Fashion MNIST, but cannot compare against \cite{hammoudeh2023feature} which is a recent preprint and does not have code.}

\section{Conclusion, Limitations and Future Work}
In this work, we developed a theoretical to exploit properties of the data distribution for robustness against sparse adversarial attacks. We showed that data \red{localization} -- the property that a data distribution $p$ places most of its mass on very small volume sets in the input space -- characterizes the existence of a $\ell_0$-robust classifier for $p$. Following this theory, we developed a defense against sparse adversarial attacks, and derived a corresponding robustness certificate. We showed that this certificate 
%is tight, and 
empirically improves upon existing state-of-the-art in several broad regimes. 

The primary limitation of our work is the difficulty in efficiently learning classifiers that have axis-aligned decision regions. While we are able to successfully employ several optimization tricks for datasets like MNIST and Fashion MNIST, the task becomes harder on more complicated datasets, \emph{even though the geometry required for the underlying data-distribution remains the same due to our general theoretical results}. These optimization difficulties mostly stem from the strict requirement of axis-aligned boxes for our distance computation in \cref{lem:l0dist}. In the future, we hope to trade-off efficiency in the distance computation in favor of richer decision boundaries that can be learnt efficiently and generalize well.

%\section*{Broader Impact}
%%This paper presents work whose goal is to advance the field of Machine Learning. There are many potential societal consequences of our work, none which we feel must be specifically highlighted here.
%This paper studies the theoretical basis for the existence of classifiers that are robust to sparse adversarial corruptions, as well as proposes sound empirical defenses. Such adversarial attacks remain important threats in the application of machine learning models in sensitive domains, such as healthcare, finance, autonomous driving, and more. In this way, our work should contribute to the implementation of provably robust models, reducing the risks of the implementations of these technologies in modern societies.

\bibliography{refs_l0conc}
\bibliographystyle{icml2024}

%%%%%%%%%%%%%%%%%%%%%%%%%%%%%%%%%%%%%%%%%%%%%%%%%%%%%%%%%%%%%%%%%%%%%%%%%%%%%%%
%%%%%%%%%%%%%%%%%%%%%%%%%%%%%%%%%%%%%%%%%%%%%%%%%%%%%%%%%%%%%%%%%%%%%%%%%%%%%%%
% APPENDIX
%%%%%%%%%%%%%%%%%%%%%%%%%%%%%%%%%%%%%%%%%%%%%%%%%%%%%%%%%%%%%%%%%%%%%%%%%%%%%%%
%%%%%%%%%%%%%%%%%%%%%%%%%%%%%%%%%%%%%%%%%%%%%%%%%%%%%%%%%%%%%%%%%%%%%%%%%%%%%%%
\newpage
\appendix
\onecolumn
\section{Auxilliary Lemmas and Proofs} \label{app:proofs}
\begin{lemma}[Properties of expansion and contraction, extending \citet{pal2023concentration}]
    \label{lem:expandcontract}
    For a distance $d$, set $A \subseteq [0, 1]^n$, define $A^{+\epsilon} = \{x \in [0, 1]^n \colon {\rm dist}_d(x, A) \leq \epsilon\}$, and $A^{-\epsilon} = \{ x \in [0, 1]^n \colon B_d(x, \epsilon) \subseteq A\}$. Then, for $N, O \subseteq [0, 1]^n$, we have 
    \begin{enumerate}
        \item $(N \cap O)^{-\epsilon} = N^{-\epsilon} \cap O^{-\epsilon}$
        \item $(N^c)^{-\epsilon} = (N^{+\epsilon})^c$, where $c$ denotes complement in $[0, 1]^n$
        \item $(N \setminus O)^{-\epsilon} = N^{-\epsilon} \setminus O^{+\epsilon}$
        \item $(N \cup O)^{+\epsilon} = N^{+\epsilon} \cup O^{+\epsilon}$
        \item $(N^{+\epsilon_1})^{+\epsilon_2} \subseteq N^{+(\epsilon_1 + \epsilon_2)}$
    \end{enumerate}
\end{lemma}

\begin{proof} The first four assertions of this Lemma are standard results in mathematical morphology, dealing with the erosion and dilation of sets, and are reproduced here from \citet{pal2023concentration} for clarity. 

\begin{enumerate}
\item Let $M = N \cap O$. 
\begin{align*}
M^{-\epsilon} &= \{x \colon x \in M, B_d(x, \epsilon) \subseteq M\} \\
&= \{x \colon x \in N, x \in O, B_d(x, \epsilon) \subseteq N, B_d(x, \epsilon) \subseteq O\} = N^{-\epsilon} \cap O^{-\epsilon}.
\end{align*}

\item Let $M = N^c$. 
\begin{align*}
M^{-\epsilon} &= \{x \colon x \in M, B_d(x, \epsilon) \subseteq M\} = \{x \colon x \not \in N, B_d(x, \epsilon) \subseteq N^c\} \\
&= \{x \colon x \not \in N, \forall x' \in B_d(x, \epsilon) \  x' \not \in N\} \\
&= \{x \colon \forall x' \in B_d(x, \epsilon) \  x' \not \in N\} \\
\implies (M^{-\epsilon})^c &= \{x \colon \exists x' \in B_d(x, \epsilon) \ x' \in N\} \\
&= N^{+\epsilon}.
\end{align*}

\item Let $M = N \setminus O$, we have $M^{-\epsilon} = (N \cap O^c)^{-\epsilon} = N^{-\epsilon} \cap (O^c)^{-\epsilon}$ by Property 1, and then $N^{-\epsilon} \cap (O^c)^{-\epsilon} = N^{-\epsilon} \cap (O^{+\epsilon})^c$ by Property 2. 

\item Let $M = N \cup O$. We have $M^c = N^c \cap O^c$. Taking $\epsilon$-contractions, and applying the first and second properties, we get 
$M^{+\epsilon} = N^{+\epsilon} \cup O^{+\epsilon}$.

\item For a set $M$, and any $\epsilon_1 \geq 0, \epsilon_2 \geq 0$, we have 
\begin{equation*}
\left(M^{+\epsilon_1}\right)^{+\epsilon_2} \subseteq M^{+(\epsilon_1 + \epsilon_2)}.
\end{equation*}
The above property can be derived from the triangle inequality applied to $d$, as 
\begin{align*}
\left(M^{+\epsilon_1}\right)^{+\epsilon_2} 
&= \{ x \colon \exists x' \in M^{+\epsilon_1}, \  d(x', x) \leq \epsilon_2 \} \\
&= \{ x \colon \exists x' \in \cX, x'' \in M, \ d(x', x) \leq \epsilon_2, d(x'', x') \leq \epsilon_1 \} \\
&\subseteq \{ x \colon \exists x'' \in M, \  d(x'', x) \leq \epsilon_2 + \epsilon_1\} = M^{+(\epsilon_1 + \epsilon_2)}.
\end{align*}
\end{enumerate}
\end{proof}

\lzerodist*
\begin{proof}
For any given $x$, recall the definition of ${\rm dist}$ to be ${\rm dist}(x, B(a, b)) = \min_{y \in B(a, b)} \|x - y\|_0$. For any $y \in B(a, b)$ we have, 
\begin{align}
\|x - y\|_0 = \sum_{i = 1}^n \1(x_i \neq y_i) \geq \sum_{i = 1}^n \1(x_i \not \in [a_i, b_i]) \1(y_i \in [a_i, b_i]) = \sum_{i = 1}^n \1(x_i \not \in [a_i, b_i]) 
\end{align}
The above implies $\min_{y \in B(a, b)} \|x - y\|_0 \geq \sum_{i = 1}^n \1(x_i \not \in [a_i, b_i])$. Then, consider $y^\star \in B(a, b)$ defined as
\begin{equation}
y^\star_i = \begin{cases}
a_i &\text{ if } x_i \not \in [a_i, b_i] \\
x_i &\text{ otherwise }
\end{cases}.
\end{equation}
We have $\|y^\star - x\|_0 = \sum_{i = 1}^n \1(x_i \not \in [a_i, b_i])$, which attains the lower bound on ${\rm dist}(x, B(a, b))$. The result follows.
\end{proof}

\thcert*
\begin{proof}
Let $x, x' \in \cX$. Define $\cB_1 = \{B_m \colon y_m = y_{m^\star}\}$, and $\cB_2 = \{B_m \colon y_m \neq y_{m^\star}\}$. Further, define $\bar d_1, \bar d_2$ as
\begin{equation*}
d_1(x') = \min_{B \in \cB_1} {\rm dist}(x', B), \quad d_2(x') = \min_{B \in \cB_2} {\rm dist}(x', B),
\end{equation*}
Our goal would be to demonstrate that as long as $\|x - x'\|_0 < {\rm margin}(x)/2$, we have $d_2(x') > d_1(x')$, implying that the prediction remains the same at $x'$. Consider any $B \in \cB_2$, and apply the triangle inequality to get 
\begin{align}
{\rm dist}(x', B) + \|x - x'\|_0 &\geq {\rm dist}(x, B), \label{eq:triangle1} 
\end{align}
where \eqref{eq:triangle1} can be seen as 
\begin{align}
{\rm dist}(x', B) + \|x - x'\|_0 = \min_{y \in B} \|y - x'\|_0 + \|x' - x\|_0 \geq \min_{y \in B} \|y - x\|_0 = {\rm dist}(x, B).
\end{align}
Further, taking a minimum on both sides of \eqref{eq:triangle1} over all $B \in \cB_2$ leads to
\begin{align}
d_2(x') + \|x - x'\|_0 \geq d_2 \label{eq:s1}
\end{align}
Similarly, consider any $B \in \cB_1$, and apply the triangle inequality to get
\begin{align}
{\rm dist}(x, B) + \|x - x'\|_0 \geq {\rm dist}(x', B), \label{eq:triangle2}. 
\end{align}
Taking a minimum over both sides of \eqref{eq:triangle2} over all $B \in \cB_1$ leads to
\begin{align}
d_1 + \|x - x'\|_0 \geq d_1(x'). \label{eq:s2}
\end{align}
Adding \eqref{eq:s1} and \eqref{eq:s2}, we have
\begin{align}
d_2(x') - d_1(x') + 2 \|x - x'\|_0 &\geq d_2 - d_1 \\
\implies d_2(x') - d_1(x') &\geq {\rm margin}(x) - 2 \|x - x'\|_0, 
\end{align}
from where we can see that $d_2(x') - d_1(x') > 0$ whenever $\|x - x'\|_0 < {\rm margin}(x) / 2$, as required.
\end{proof}

\section{Additional Empirical Comparison} \label{app:addlcomp}
We produce an additional comparison to $\ell_0$ certificates in \cite{hammoudeh2023feature}. Since there is no publicly available code for this method, we compare our method against the numbers reported in \citet[Table 27]{hammoudeh2023feature}. In \cref{fig:compdet}, we compare against the method ``FPA with run-off elections'' reported in \cite{hammoudeh2023feature}. This method uses a more complicated aggregation scheme on top of \cref{eq:ensemble} to obtain improved certificates. Nevertheless we observe that \textsc{Box-NN} improves upon the median certified robustness for all methods in all hyper parameter settings. 

\begin{figure}[ht!]
\centering
\includegraphics[width=0.4\textwidth]{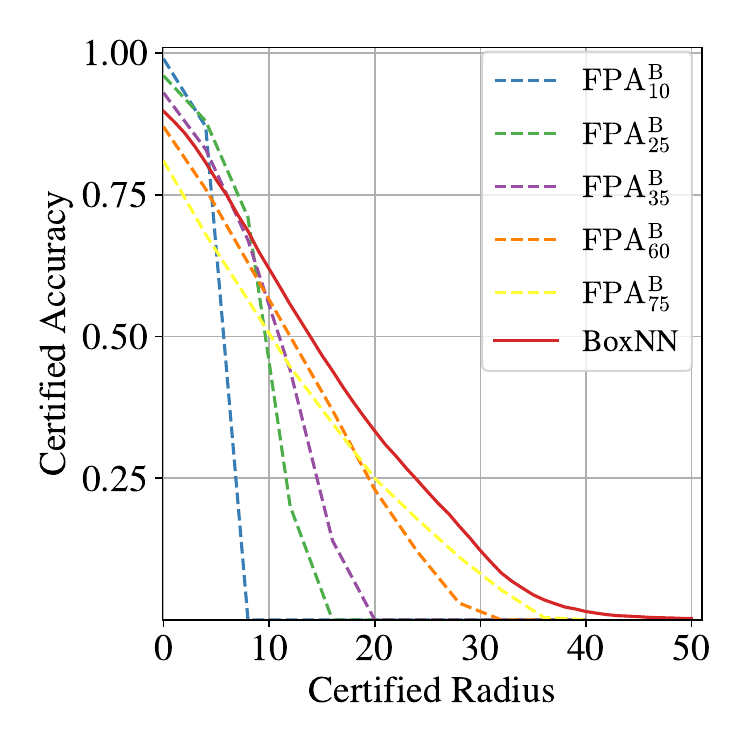}
\caption{Comparison of a deterministic certificate \cite{hammoudeh2023feature} (dotted lines) to our method  \textsc{Box-NN} (red line) on the MNIST dataset. The dotted lines correspond to different settings for the hyperparameter $\rho$. Details are mentioned in main text.}
\label{fig:compdet}
\end{figure}

%%%%%%%%%%%%%%%%%%%%%%%%%%%%%%%%%%%%%%%%%%%%%%%%%%%%%%%%%%%%%%%%%%%%%%%%%%%%%%%
%%%%%%%%%%%%%%%%%%%%%%%%%%%%%%%%%%%%%%%%%%%%%%%%%%%%%%%%%%%%%%%%%%%%%%%%%%%%%%%

\end{document}